\documentclass[10pt,twocolumn,letterpaper]{article}

\usepackage[pagenumbers]{cvpr} 

\definecolor{cvprblue}{rgb}{0.21,0.49,0.74}
\usepackage[pagebackref,breaklinks,colorlinks,allcolors=cvprblue]{hyperref}

\usepackage{colortbl}
\usepackage{booktabs}
\usepackage{subcaption}
\usepackage{xcolor} 

\usepackage{amsmath}
\usepackage{amssymb, amsfonts}
\usepackage{mathtools}
\usepackage{amsthm}

\usepackage[accsupp]{axessibility}  

\theoremstyle{plain}
\newtheorem{theorem}{Theorem}[section]

\theoremstyle{definition}
\newtheorem{definition}[theorem]{Definition}
\newtheorem{assumption}[theorem]{Assumption}
\theoremstyle{remark}

\def\confName{CVPR}
\def\confYear{2026}

\title{Deconstructing the Failure of Ideal Noise Correction: A Three-Pillar Diagnosis}
\newcommand{\fc}{\mathtt{FC}}
\newcommand{\nc}{\mathtt{NC}}

\newcommand{\corr}{\mathtt{corr}}

\author{%
Chen Feng\textsuperscript{1,}\thanks{Corresponding author: \url{https://mrchenfeng.github.io}.},\hspace{0.1cm}
Zhuo Zhi\textsuperscript{2},\hspace{0.1cm}
Zhao Huang\textsuperscript{3},\hspace{0.1cm}
Jiawei Ge\textsuperscript{4},\hspace{0.1cm}
Ling Xiao\textsuperscript{5} \\
Nicu Sebe\textsuperscript{6},\hspace{0.1cm}
Georgios Tzimiropoulos\textsuperscript{4},\hspace{0.1cm}
Ioannis Patras\textsuperscript{4} \\
\textsuperscript{1}Queen's University Belfast \quad
\textsuperscript{2}University College London \quad
\textsuperscript{3}University of Aberdeen \\
\textsuperscript{4}Queen Mary University of London \quad
\textsuperscript{5}Hokkaido University \quad
\textsuperscript{6}University of Trento
}

\begin{document}
\maketitle
\begin{abstract}
Statistically consistent methods based on the noise transition matrix ($T$) offer a theoretically grounded solution to Learning with Noisy Labels (LNL), with guarantees of convergence to the optimal clean-data classifier. In practice, however, these methods are often outperformed by empirical approaches such as sample selection, and this gap is usually attributed to the difficulty of accurately estimating $T$. The common assumption is that, given a perfect $T$, noise-correction methods would recover their theoretical advantage.
In this work, we put this longstanding hypothesis to a decisive test. We conduct experiments under idealized conditions, providing correction methods with a \emph{perfect, oracle transition matrix}. Even under these ideal conditions, we observe that these methods still suffer from performance collapse during training. This compellingly demonstrates that the failure is not fundamentally a $T$-estimation problem, but stems from a more deeply rooted flaw.
To explain this behaviour, we provide a unified analysis that links three levels: macroscopic convergence states, microscopic optimisation dynamics, and information-theoretic limits on what can be learned from noisy labels. Together, these results give a formal account of why ideal noise correction fails and offer concrete guidance for designing more reliable methods for learning with noisy labels.
\end{abstract}    
\section{Introduction}
\label{sec:introduction}

Learning with noisy labels (LNL) is a fundamental and pervasive challenge in modern machine learning. Label inaccuracies, often arising from human error or automated annotation, can seriously bias model training and weaken generalisation. To address this, a prominent family of solutions, known as \emph{statistically consistent} methods, offers theoretical guarantees of convergence to the optimal classifier one would obtain from clean data. Central to this class is the \emph{noise transition matrix} ($T$) framework~\cite{loss_correction,loss_correction2,noiseadaptation}, which explicitly models the label corruption process to enable principled loss correction or risk estimation.

\begin{figure*}[t]
\centering
\resizebox{\textwidth}{!}{
\begin{minipage}{0.49\textwidth}
\centering
\includegraphics[width=\linewidth]{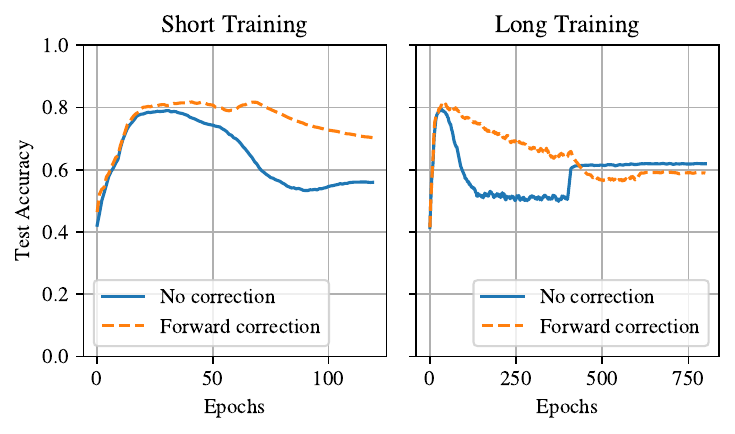}
\subcaption{Test accuracy on CIFAR-10 with 50\% symmetric noise.}
\label{fig:cifar10_1}
\end{minipage}
\hfill
\begin{minipage}{0.49\textwidth}
\centering
\includegraphics[width=\linewidth]{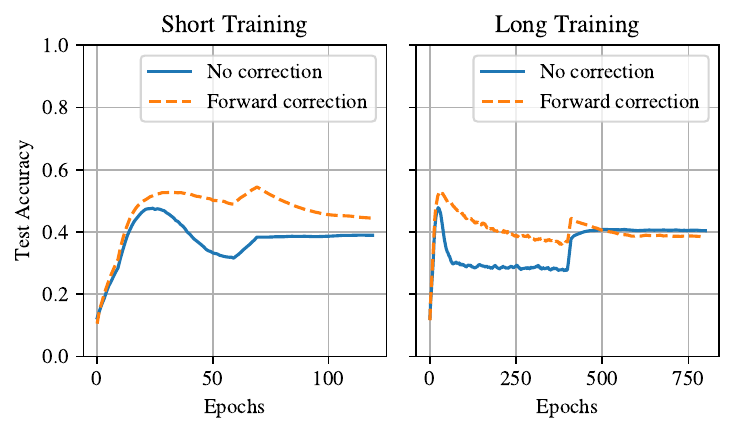}
\subcaption{Test accuracy on CIFAR-100 with 50\% symmetric noise.}
\label{fig:cifar100_1}
\end{minipage}
}
\caption{Comparison of test accuracy across different training durations. Even when provided with the ground-truth transition matrix (Oracle $T$), the model's performance collapses after an initial peak. Prolonged training sees the correction method's efficacy degrade, ultimately offering no benefit over the uncorrected baseline. 
}
\label{fig:ideal_failure}
\end{figure*}

Yet, despite their theoretical elegance, a persistent paradox remains: in practice these principled methods are often outperformed by empirically successful yet \emph{statistically approximate} approaches, such as sample selection~\cite{coteaching, dividemix}. A superficial explanation for this performance gap is that many state-of-the-art sample selection methods are complex hybrid frameworks, integrating powerful techniques such as model ensembling~\cite{coteaching, coteaching+}, semi-supervised learning~\cite{dividemix}, and advanced data augmentations~\cite{augdesc, selcl}. However, these components are conceptually orthogonal to the core LNL setting and could, in principle, be combined with either family of methods.

This shifts the focus to a more dominant hypothesis within the LNL community: the failure of noise correction methods is attributed almost exclusively to the difficulty of accurately estimating the noise transition matrix $T$~\cite{xia2019anchor,yao2020dual,liu2023identifiability}. The widely held belief is that, if a sufficiently accurate $T$ were available, this paradox would be resolved and these theoretically motivated methods would regain their expected advantage.

In this work, we put this longstanding hypothesis to a decisive test. We conduct controlled experiments under idealised conditions, providing the correction methods with a perfect, oracle transition matrix. To isolate the core mechanism, our experimental setup is deliberately minimalistic, removing auxiliary components that are often bundled with such methods. The results, shown in~\Cref{fig:ideal_failure}, challenge the conventional view. Even with a perfect $T$, Forward Correction (FC) exhibits a pronounced rise-and-fall dynamic and eventually converges to the same poor performance as direct training without any loss correction. This is strong evidence that the failure of noise correction is not fundamentally a problem of estimating $T$; instead, it reflects a deeper limitation of the corrected objective itself.

In this paper, we conduct a rigorous investigation to diagnose the root causes of this \emph{ideal tool failure}. Our objective is explicitly \emph{not} to propose yet another correction heuristic, but rather to provide a comprehensive theoretical analysis that systematically explains \emph{why} these principled methods fail even when supplied with perfect information. To this end, our analysis departs from traditional asymptotic bounds and moves beyond common simplifying assumptions such as class-conditional noise (CCN) or deterministic posteriors, so that the conclusions apply to a broad range of practical settings. Our diagnosis is built on three complementary results:
\begin{enumerate}
    \item A \textbf{macroscopic analysis} contrasting the \textit{Ideal Fitted Case} with the \textit{Empirical Overfitted Case} to characterise the final convergence states.
    \item A \textbf{microscopic analysis} examining the optimisation dynamics and per-sample gradients to expose inherent instabilities.
    \item A \textbf{fundamental analysis} adopting an information-theoretic view to quantify the unavoidable information loss introduced by the noise corruption process.
\end{enumerate}
Collectively, these findings uncover the mechanisms that drive this paradox and provide the community with a rigorous account of why statistically consistent methods falter, pointing towards design principles for more reliable LNL solutions.

\section{Related Work}
\label{sec:related_works}


\paragraph{Noise Transition Matrix (T-Matrix).}
This paradigm seeks to establish strong statistical foundations, aiming to guarantee \emph{classifier consistency} asymptotically~\cite{,chu2023topology,chu2025deep,chu2025tmi,li2025hybrid,li2024size,bao2025towards,zhang2024learnable,zhang2025exploit,zhang2025molebridge,li2026dontletinformationslip}. Seminal works introduced Forward and Backward Correction to rectify loss or risk using the matrix $T$~\cite{loss_correction,loss_correction2,natarajan2013learning,liu2025fedadamw,liu2025fedmuon,liuimproving,jiang2025towards,yang2024regulating,xie2025seqgrowgraph}. Consequently, the field dedicated extensive effort to the accurate estimation of $T$, evolving from early \emph{anchor point assumptions}~\cite{liu2015classification,xia2019anchor} to advanced structure-based techniques (e.g., volume minimization, clusterability~\cite{yao2020dual,li2021provably,zhu2021clusterability}). Recently, attention shifted to the complexity of Instance-Dependent Noise (IDN)~\cite{chen2021INDnoise, xia2020part}, though identifying IDN remains generally ill-posed without strong inductive biases~\cite{liu2023identifiability}.


\paragraph{Robust Loss Functions.}
A parallel theoretical approach focuses on designing inherently tolerant loss functions, such as GCE, MAE, and SCE~\cite{generalized_cross_entropy,ghosh2015making,symmetric_cross_entropy}. Unlike matrix-based methods, these approaches bypass the explicit modeling of the noise process via $T$, achieving asymptotic Bayes consistency by imposing specific structural constraints, notably the symmetry condition.

\paragraph{Empirically-Driven Sample Selection.}
These methods prioritize empirical efficacy using heuristics (e.g., the `small-loss' hypothesis) to filter clean samples. Pioneering works like Co-teaching~\cite{coteaching,coteaching+} utilized dual networks for cross-filtering. This evolved into a diverse array of filtering and regularization techniques based on prediction confidence, graph consistency, and feature-space topology~\cite{dividemix,selfie,propmix,topofilter,ngc,Feng2022SSR,moit,liu2024fedbcgd,qu2024conditional,qu2025end,liu2025consistency,liu2025dp,Feng2025OpenSet,Feng2024NoiseBox,Feng2024CLIPCleaner}. Recent state-of-the-art frameworks further integrate semi-supervised and contrastive learning~\cite{dividemix,selcl,karim2022unicon_selection,Feng2023MaskCon,Feng2022adaptive}. While adaptable, they fundamentally rely on the implicit regularization of deep networks rather than formal statistical correction.
\section{Problem Setting and Preliminaries}
\label{sec:preliminaries}

In this section, we formalize the framework of LNL, establish the necessary notation, and define the rigorous assumptions that underpin our theoretical analysis.

\subsection{Problem Formulation}
\label{subsec:formulation}

Let $\mathcal{X} \subseteq \mathbb{R}^d$ be the feature space and $\mathcal{Y} = \{1, \dots, K\}$ be the label space. Let $\Delta^{K-1}$ denote the probability simplex over $K$ classes. We assume the data is generated from a joint distribution $\mathcal{D}$ over $\mathcal{X} \times \mathcal{Y}$. For any instance $x \in \mathcal{X}$, there exists a latent, clean label $Y$ drawn from the conditional posterior $P(Y|X=x)$.
However, in the LNL setting, the clean label $Y$ is unobserved. Instead, the learner is presented with a noisy dataset $S_n = \{(x_i, y^n_i)\}_{i=1}^N$, where $y^n_i \in \mathcal{Y}$ represents the observed noisy label.

The label corruption process is governed by an \emph{instance-dependent noise transition matrix} $T: \mathcal{X} \to [0, 1]^{K \times K}$. The entry $T_{ij}(x)$ represents the probability of the true label $i$ flipping to the noisy label $j$ given the instance $x$:
\begin{equation}
    T_{ij}(x) \triangleq P(Y^n=j \mid Y=i, X=x).
\end{equation}
This formulation relates the clean posterior $\eta(x) \triangleq [P(Y=k|x)]_{k=1}^K$ to the noisy posterior $\eta^n(x) \triangleq [P(Y^n=k|x)]_{k=1}^K$ via the linear transformation:
\begin{equation}
\label{eq:transition}
    \eta^n(x) = T(x)^\top \eta(x).
\end{equation}
\begin{assumption}[Diagonal-Dominant Noise]\label{assumption:diagonal_dominant} We assume $T(x)$ is diagonally dominant ($T_{ii}(x)>T_{ij}(x),j\neq i$) for all $x \in \mathcal{X}$. Crucially, we \emph{do not} rely on the simplifying assumption of Class-Conditional Noise (CCN); our formulation thus holds for the general, real-world Instance-Dependent Noise (IDN) setting.
\end{assumption}

\subsection{Learning Objectives}
\label{subsec:objectives}

The goal is to learn a scoring function $f: \mathcal{X} \to \mathbb{R}^K$ from a hypothesis class $\mathcal{F}$. The function $f(x)$ outputs a logit vector, which induces a predicted probability distribution $\hat{p}(x) = \text{softmax}(f(x)) \in \Delta^{K-1}$.

We adopt the \textit{Cross-Entropy (CE) loss} as the base loss function $\ell$. The ultimate objective is to minimize the \emph{expected clean risk}:
\begin{equation}
    R^{c}(f) \triangleq \mathbb{E}_{(X,Y) \sim \mathcal{D}} [\ell(f(X), Y)],
\end{equation}
where $\ell(f(x), y) \triangleq -\log(\hat{p}_y(x))$.
Since clean labels are inaccessible, a naive approach is to minimize the \emph{empirical noisy risk} defined on $S_n$:
\begin{equation}
    \widehat{R}_{\nc}^{n}(f) \triangleq \frac{1}{N}\sum_{i=1}^{N} \ell(f(x_i), y^n_i).
\end{equation}
We refer to this as the \textbf{No Correction (NC)} baseline. Asymptotically, this converges to the expected noisy risk $R_{\nc}^{n}(f) = \mathbb{E}_{(X,Y^n)}[\ell(f(X), Y^n)]$, which is a biased estimator of $R^c(f)$ and typically yields a suboptimal classifier.

\subsection{Risk Correction}
\label{subsec:risk_correction}

To mitigate the bias inherent in $\widehat{R}_{\nc}^{n}(f)$, \textit{statistically consistent} methods construct a surrogate loss function. 
The gold standard for such methods is achieving \emph{Bayes-classifier consistency}:
$    \arg\min_{f \in \mathcal{F}} R_{\corr}^n(f) = \arg\min_{f \in \mathcal{F}} R^{c}(f).$
While two primary paradigms exist for risk correction—Forward Correction and Backward Correction\footnote{\textbf{Backward Correction (BC)} utilizes the inverse matrix $T^{-1}$ to provide an unbiased risk estimator ($R_{\text{BC}}^n = R^c$). However, we exclude BC from our primary analysis as it suffers from severe numerical instability: $T^{-1}$ typically introduces negative weights, rendering the loss unbounded below and prone to optimization collapse~\cite{kiryo2017positive,chou2020unbiased}.}—we focus our analysis on Forward Correction due to its optimization stability and widespread adoption.

\paragraph{Forward Correction (FC).}
FC explicitly models the noise generation process within the loss function. Instead of correcting the loss gradients (as in BC), FC corrects the model's prediction. Let $\hat{p}(x)$ be the model's prediction for the clean class posterior. The predicted probability for the \emph{noisy} class is then given by $T(x)^\top \hat{p}(x)$. The Forward Corrected loss is defined as:
\begin{equation}
    \ell_{\fc}(f(x), y^n; x) \triangleq -\log \left( \left[ T(x)^\top \hat{p}(x) \right]_{y^n} \right).
\end{equation}
Minimizing the expected risk under $\ell_{\fc}$ is proven to be classifier-consistent with the clean risk $R^c(f)$, provided $T(x)$ is known or accurately estimated.

\subsection{Evaluation Metrics}
\label{subsec:evaluation_metrics}

To rigorously assess model performance, we employ two complementary metrics:

\begin{itemize}
    \item \textbf{Classification Accuracy (ACC):} Measures the discriminative power of the model with respect to the ground truth labels:
    \[
        \mathrm{ACC}(f) = \mathbb{E}_{(X,Y)} \bigl[ \mathbb{I}(Y = Y^f(X)) \bigr],
    \]
    where $Y^f(X) = \arg\max_k \hat{p}_k(X)$.

    \item \textbf{Expected Calibration Error (ECE):} Measures the reliability of the model's confidence. Let $C^f(X) = \max_k \hat{p}_k(X)$ be the confidence of the prediction. ECE quantifies the expected deviation between confidence and true accuracy:
    \[
        \mathrm{ECE}(f) = \mathbb{E}_{C} \Bigl[ \bigl| P(Y = Y^f(X) \mid C^f(X)=C) - C \bigr| \Bigr].
    \]
\end{itemize}

\section{Theoretical Analysis of Noise Correction}
\label{sec:rethinking_T}
In this section, we dismantle the performance paradox of Forward Risk Correction. A substantial body of literature has established that Forward Correction is \textit{asymptotically consistent}~\cite{loss_correction,natarajan2013learning}. Formally, let $\mathcal{R}^n_{\fc}(f)$ be the population risk under Forward Correction. It holds that the minimizer $f^* = \arg\min_{f} R^n_{\fc}(f)$ coincides with the Bayes-optimal classifier for the clean distribution, i.e., $R^c(f^*) = \min_f R^c(f)$.

However, as demonstrated empirically in \Cref{fig:ideal_failure}, these theoretical guarantees do not translate to robustness in practice. Even when equipped with the oracle transition matrix $T(x)$, Forward Correction exhibits a characteristic ``overfitting" trajectory: accuracy peaks early but degrades significantly during prolonged training, eventually regressing to the performance of the uncorrected baseline (NC).

We posit that this discrepancy is strictly a \emph{finite-sample phenomenon}. Standard asymptotic analysis focuses on the convergence of the population risk $R(f) \to R^*$. In contrast, deep learning optimization is dominated by the \emph{empirical risk} $\widehat{R}(f)$. High-capacity networks are prone to memorizing the finite noisy dataset, driving $\widehat{R}(f) \to \min_f\widehat{R}(f)$ even when this empirical minimizer corresponds to a poor population-level solution.

To dissect this failure, our analysis departs from traditional statistical bounds and proceeds in three complementary stages:
\begin{enumerate}
\item \textbf{Macroscopic Analysis:} We characterize the terminal convergence states, contrasting the ideal population minimizer with the empirical minimizer of the finite noisy dataset.
\item \textbf{Microscopic Analysis:} We investigate the optimization dynamics via single-sample gradients to pinpoint the source of memorization.
\item \textbf{Fundamental Analysis:} We adopt an information-theoretic perspective to quantify the irreducible information loss induced by the noise channel $T(x)$.
\end{enumerate}

\subsection{Theoretical Framework}
\label{subsec:theoretical_framework}


To evaluate the alignment between the learning objective and the ground truth, 
we define the following optimal predictors:

\begin{itemize}
\item \emph{Clean Bayes-Optimal Classifier $Y^*(x)$:} The most probable class according to the latent clean posterior:
\[
Y^*(x) \triangleq \arg\max_{k \in \mathcal{Y}} P(Y = k \mid X=x).
\]

\item \emph{Inherent Uncertainty $\delta(x)$:} The probability mass assigned to non-optimal classes (aleatoric uncertainty):
\[
\delta(x) \triangleq 1 - P(Y = Y^*(x) \mid X=x).
\]
Note that unless the clean distribution is deterministic, $\delta(x) > 0$.

\item \emph{Noisy Bayes-Optimal Classifier $\tilde{Y}^*(x)$:} The most probable class according to the observed noisy posterior:
\[
\tilde{Y}^*(x) \triangleq \arg\max_{k \in \mathcal{Y}} P(Y^n = k \mid X=x).
\]
\end{itemize}

The core difficulty of the LNL problem lies in the misalignment between the clean and noisy objectives. We formalize this by partitioning the input space based on the consistency of these predictors.

\begin{definition}[Population-Level Consistency Partition]
\label{def:consistency_partition}
We partition the input space $\mathcal{X}$ into two disjoint sets:
\begin{align*}
\mathcal{X}_{\text{correct}} &= \{x \in \mathcal{X} \mid \tilde{Y}^*(x) = Y^*(x)\} \\
\mathcal{X}_{\text{error}} \;\; &= \{x \in \mathcal{X} \mid \tilde{Y}^*(x) \neq Y^*(x)\}
\end{align*}
\end{definition}

\noindent\textbf{Interpretation.} 
$\mathcal{X}_{\text{correct}}$ represents the ``benign" regime where the noise, despite corrupting individual samples, preserves the dominant class in the posterior distribution. $\mathcal{X}_{\text{error}}$ represents the ``malignant" regime where the noise channel $T(x)$ is sufficiently strong—relative to the clean margin—to flip the optimal decision boundary. Intuitively, samples with high inherent uncertainty $\delta(x)$ (i.e., ambiguous samples close to the decision boundary) are statistically prone to falling into $\mathcal{X}_{\text{error}}$.

\paragraph{Comparison with Prior Art.}
Most theoretical works before often assume a deterministic clean posterior (i.e., $\delta(x) \equiv 0$). Under this simplification, combined with Assumption~\ref{assumption:diagonal_dominant}, it strictly holds that $\mathcal{X}_{\text{error}} = \emptyset$. By acknowledging that $\delta(x) \ge 0$, our framework explicitly accounts for the $\mathcal{X}_{\text{error}}$ region and thus generalizes to broader scenarios. 
As we will show, the behavior of correction methods on this often-overlooked set is crucial.

\subsection{Macroscopic Analysis: Ideal vs. Overfitted States}
\label{subsec:macroscopic_analysis}

We begin our analysis by examining the two ``end states'' of the learning process. The central paradox of LNL—its theoretical success (asymptotic consistency) versus its practical failure (overfitting)—is best validated by contrasting these two extremes:

\begin{enumerate}
\item \textbf{The Ideal Fitted Case ($R(f) \to R^*)$}: This represents the \textit{asymptotic promise} (e.g., $N \to \infty$). It allows us to establish the theoretical \textit{optimum} of correction methods (i.e., ``what we hope to achieve'').
\item \textbf{The Empirical Overfitted Case ($\widehat{R}(f) \to 0$)}: This represents the \textit{finite-sample reality} of high-capacity deep networks. It models the failure mode where the model memorizes the noisy dataset (i.e., ``what we actually observe'').
\end{enumerate}

\subsubsection{Analysis of the Ideal Fitted Case}
\label{subsubsec:ideal_case}

We first analyze the scenario where the optimizer finds the true population minimizer $f^*$.
As established in \Cref{subsec:risk_correction}, Forward Correction (FC) is consistent with the clean distribution, whereas No Correction (NC) targets the noisy distribution. Their optimal population-level predictions $\hat{p}^*(x) = \text{softmax}(f^*(x))$ match the true posteriors:
\begin{equation*}
\hat{p}_{\nc}^*(x) = \eta^n(x), \quad \hat{p}_{\fc}^*(x) = \eta(x),
\end{equation*}
where $\eta(x)$ and $\eta^n(x)$ are the clean and noisy posteriors defined in \Cref{subsec:formulation}.
Based on these posteriors, we characterize their performance using the partition $(\mathcal{X}_{\text{correct}}, \mathcal{X}_{\text{error}})$ from Definition~\ref{def:consistency_partition}.

\begin{theorem}[Optimality and Consistency Gap under Ideal Fitting]
\label{theorem:ideal_performance}
Let $f_{\nc}$ and $f_{\fc}$ be the ideal population minimizers of the No Correction and Forward Correction risks, respectively.
\begin{enumerate}
\item[(a)] \textbf{Accuracy (ACC):} FC achieves the optimal Bayes accuracy. The accuracy gap $\Delta$ between FC and NC is non-negative and localized to the error set $\mathcal{X}_{\text{error}}$:
\begin{itemize}
    \item $\mathrm{ACC}(f_{\fc}) = 1 - \mathbb{E}_X[\delta(X)].$
\item $\mathrm{ACC}(f_{\fc}) - \mathrm{ACC}(f_{\nc}) = \Delta.$
\end{itemize}
where
\[
 \Delta \geq P(\mathcal{X}_{\text{error}}) \cdot \mathbb{E}_{X \mid \mathcal{X}_{\text{error}}}\Big[ \max\big(0, 1-2\delta(X)\big) \Big] \ge 0.
\]

\item[(b)] \textbf{Calibration (ECE):} FC achieves perfect calibration with respect to the clean labels, whereas NC is inherently miscalibrated:
$\mathrm{ECE}(f_{\fc}) = 0,  \mathrm{ECE}(f_{\nc}) > 0$.
\end{enumerate}
\end{theorem}
The proof is provided in Appendix \ref{app:proof_ideal_case}. \Cref{theorem:ideal_performance} provides critical insights that connect our theory to the empirical phenomena observed in \Cref{fig:ideal_failure}:

\paragraph{The ``Early Peak".} 
\Cref{theorem:ideal_performance} proves FC's superiority ($\Delta \ge 0$) in the ideal fitted state. This directly explains the "early peak" phenomenon observed in \Cref{fig:ideal_failure}. This peak aligns with the well-documented ``memorization effect,'' where models learn clean, generalizable patterns first. During this initial phase, the network's state approximates the ideal population minimizer ($f_{\fc}^*$). Since FC is provably superior in this state, it exhibits a clear advantage, successfully correcting the decision boundary before overfitting causes its performance to collapse.

\paragraph{The Role of Inherent Uncertainty.}
The accuracy gain $\Delta$ is structurally tied to the inherent uncertainty $\delta(X)$. As defined, the error set $\mathcal{X}_{\text{error}}$—where correction is necessary—is predominantly populated by samples with high aleatoric uncertainty (i.e., ambiguous, low-margin samples). This theoretical link directly explains the empirical disparity seen in \Cref{fig:ideal_failure}. CIFAR-100, with its more fine-grained classes, possesses higher inherent uncertainty ($\delta(X)$) than CIFAR-10. This implies a larger $\mathcal{X}_{\text{error}}$ region, which, according to our theorem, translates directly to a larger potential accuracy gain ($\Delta$). This perfectly matches the more pronounced improvement FC over NC on CIFAR-100 in the early peak.

\subsubsection{Analysis of the Empirical Overfitted Case}
\label{subsubsec:overfit_case}

We now pivot to the practical failure mode. Modern deep networks typically possess sufficient capacity to drive the empirical risk $\widehat{R}(f)$ to its global minimum, a phenomenon widely recognized as \textit{memorization}. 
To make the analysis of this state tractable, we adopt a simplifying analytical framework regarding the training data. We consider a hypothetical \emph{``full-support, single-label'' dataset}:
\begin{quote}
\textit{We assume the training set $S_n$ covers the input space $\mathcal{X}$ (full support) and contains exactly one sample pair $(x, y^n)$ for every $x$, where the noisy label $y^n$ is a single random draw from the conditional distribution $P(Y^n \mid X=x)$.}
\end{quote}

The implication of this "full-support" framework is that the global minimization decouples into $N$ independent, per-sample problems. We now analyze the target $\hat{p}(x)$ that minimizes the single-sample loss for $(x, y^n)$.

\begin{itemize}[leftmargin=*,topsep=0pt,parsep=0pt,itemsep=3pt]
    \item \textbf{No Correction (NC):} Minimizing the standard CE loss $\ell_{\nc} = -\log(\hat{p}_{y^n})$ trivially forces the model to memorize the noisy label: $\hat{p}_{\nc}(x) = \mathbf{e}_{y^n}$.
    
    \item \textbf{Forward Correction (FC):} The objective is to minimize $\ell_{\fc}$. As we rigorously show in Appendix~\ref{app:proof_fc_collapse}, this loss is minimized by collapsing the prediction to a hard vertex corresponding to the largest entry in the $y^n$-th column of $T(x)$:
    \begin{equation}
    \label{eq:fc_collapse}
    \hat{p}_{\fc}(x) = \mathbf{e}_{k^*_{\fc}(x)},
    \end{equation} 
    \end{itemize}
where $k^*_{\fc}(x) = \arg\max_k T_{k, y^n}(x)$.

We then formulate the statistical consequences of this collapse in the following theorem. Detailed proofs are provided in Appendix \ref{supp:proof_acc_ece_fitted}.

\begin{theorem}[Accuracy Trade-off and Solution Collapse under Memorization]
\label{theorem:macroscopic_tradeoff}
Consider the Empirical Overfitted Case where the empirical risk $\widehat{R}(f)$ is minimized.
\begin{enumerate}
\item[(a)] \textbf{Accuracy (ACC):} In the first-order approximation ($\varepsilon = \mathcal{O}(\mathbb{E}_X[\delta(X)])$), the accuracy of each method is determined by its ability to map the true label $Y^*$ correctly:
\begin{itemize}
\item $\mathrm{ACC}_{\nc} \approx \mathbb{E}_X \left[ (1 - \delta(X)) T_{Y^*,Y^*}(X) \right] + \varepsilon$.
\item $\mathrm{ACC}_{\fc} \approx \mathbb{E}_X \left[ (1 - \delta(X)) C_{Y^*}(X) \right] + \varepsilon$, 
\end{itemize}
where $C_{Y^*}(X) \triangleq \sum_{j: k^*(j)=Y^*} T_{Y^*,j}(X)$, $k^*(j) = \arg\max_k T_{kj}$.
\item[(d)] \textbf{Calibration (ECE):} The prediction confidence collapses to $\hat{c}=1$ everywhere, resulting in the ECE being perfectly coupled with accuracy: $\mathrm{ECE}_{\fc/\nc} = 1 - \mathrm{ACC}_{\fc/\nc}$.
\end{enumerate}
\end{theorem}

\paragraph{Confirmed Symmetric Collapse.}
To accurately analyze the accuracy gap $\Delta \mathrm{ACC}\triangleq \mathrm{ACC}_\fc - \mathrm{ACC}_\nc$ for any general instance-dependent $T$ is analytically intractable across the entire distribution. Though, under symmetric noise, the $T$-matrix structure ensures $C_{Y^*}(X) = T_{Y^*,Y^*}(X)$. This mathematically proves that the accuracy gap $\Delta \mathrm{ACC}  \approx 0$. This confirms the trajectory observed in \Cref{fig:ideal_failure}, where the FC and NC performance curves collapse to the same suboptimal level in the prolonged training finally.

\paragraph{Calibration Collapse.}
The perfect coupling $\mathrm{ECE} = 1 - \mathrm{ACC}$ suggests that in the overfitted state, the network is not only inaccurate but also \emph{maximally miscalibrated}. The loss function collapses the soft posterior into a deterministic, one-hot vector, providing the theoretical explanation for the extreme overconfidence visible in the experimental ECE plots.


\subsection{Microscopic Analysis: The Illusion of Gradient Softening}
\label{subsec:gradient_derivations}

The remaining challenge in understanding statistical consistency lies in the \emph{transient dynamics} between the ideal asymptotic state ($R \to R^*$) and the final overfitted state ($\hat{R} \to 0$). To illuminate this optimization trajectory, we analyze the per-sample gradients.

Let $f_k(\mathbf{x})$ be the logit for class $k$, inducing the prediction probability $\hat{p}_k = \exp(f_k(\mathbf{x})) / \sum_j \exp(f_j(\mathbf{x}))$. The gradients with respect to the logits $f_k(\mathbf{x})$ are derived as follows:

\begin{enumerate}
\item \textbf{No Correction (NC) Gradient:}
The standard Cross-Entropy loss $\ell_{\nc}$ yields the gradient that pushes the logits directly toward the observed noisy label:
\[
\frac{\partial \ell_{\nc}}{\partial f_k(\mathbf{x})} = \hat{p}_k - \mathbb{I}\{y^n = k\}.
\]
\item \textbf{Forward Correction (FC) Gradient:}
The FC loss $\ell_{\fc} = -\log [T(\mathbf{x})^\top \hat{p}(\mathbf{x})]_{y^n}$ yields the corrected gradient (detailed derivation in Appendix \ref{app:proof_fc_gradient}):
\begin{equation}\label{eq:fc_gradient}
    \frac{\partial \ell_{\fc}}{\partial f_k(\mathbf{x})} = \hat{p}_k - q_k,
\end{equation}
where $q_k = \frac{T_{k,y^n}(\mathbf{x}) \hat{p}_k}{\sum_{j} T_{j,y^n}(\mathbf{x}) \hat{p}_j}$. This term $q_k$ represents the estimated reverse posterior $P(Y=k \mid Y^n=y^n, \mathbf{x})$.
\end{enumerate}

\paragraph{FC `mitigating' overfitting?}
We analyze the gradient difference term $$\Delta_k = \frac{\partial \ell_{\nc}}{ \partial f_k} - \frac{\partial \ell_{\fc}}{\partial f_k} = q_k - \mathbb{I}\{y^n = k\}.$$

\begin{itemize}
\item \textbf{Mitigation:} When $k = y^n$ (the observed label), the FC gradient provides a \textit{weaker push} (less negative magnitude) compared to NC, thereby tempering the tendency to memorize the specific noisy label.
\item \textbf{Correction:} When $k \neq y^n$, FC allows classes with high likelihood (large $T_{k,y^n}(\mathbf{x})$) to compete for probability mass by \textit{reducing penalties} on alternative classes.
\end{itemize}
Collectively, FC's gradient analysis reveals a `softening effect' that guides the model toward the theoretically ideal posterior during intermediate training stages. This mechanism provides the further explanation for the `early peak' observed in \Cref{fig:ideal_failure}.

\paragraph{The Inevitable Collapse.}
The observed softening is, however, an illusion. While the gradients initially encourage a `benign' path, the final convergence behavior\footnote{Please refer to Appendix \ref{app:proof_gradient_flow} for a gradient perspective analysis on this.} is dictated by the global empirical minimums established in \Cref{theorem:macroscopic_tradeoff}.
The initial softening of FC is merely the \textit{transient dynamic} of the optimization path as it travels toward its unique, $T$-biased, and equally pathological hard-vertex attractor $\mathbf{e}_{k^*_{\fc}}$.

\begin{figure*}[t]
    \centering
    \begin{subfigure}{0.24\textwidth}
        \centering
        \includegraphics[width=\linewidth]{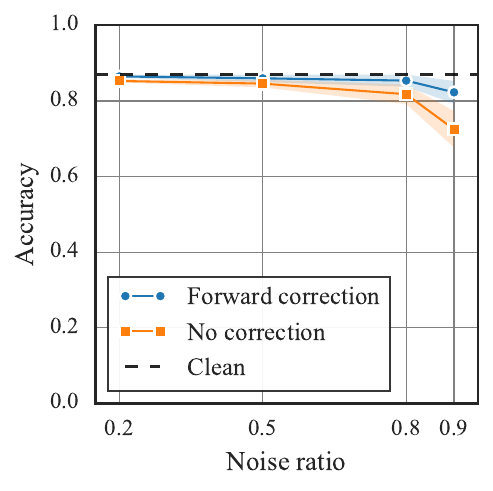}
        \caption{CIFAR-10 Accuracy}
        \label{fig:cifar10_self_acc}
    \end{subfigure}
    \begin{subfigure}{0.24\textwidth}
        \centering
        \includegraphics[width=\linewidth]{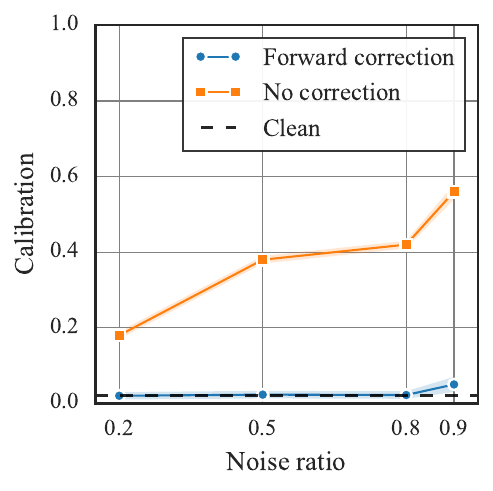}
        \caption{CIFAR-10 ECE}
        \label{fig:cifar10_self_ece}
    \end{subfigure}
    \begin{subfigure}{0.24\textwidth}
        \centering
        \includegraphics[width=\linewidth]{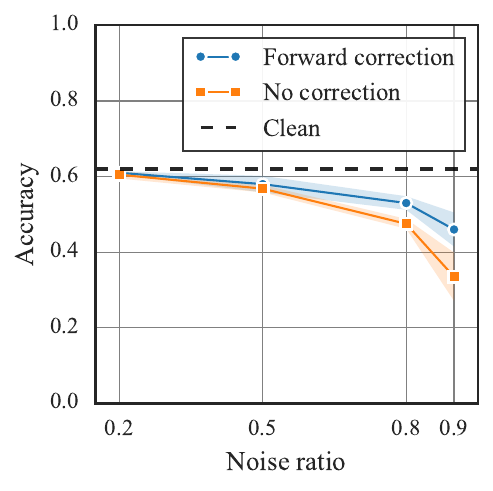}
        \caption{CIFAR-100 Accuracy}
        \label{fig:cifar100_self_acc}
    \end{subfigure}
    \begin{subfigure}{0.24\textwidth}
        \centering
        \includegraphics[width=\linewidth]{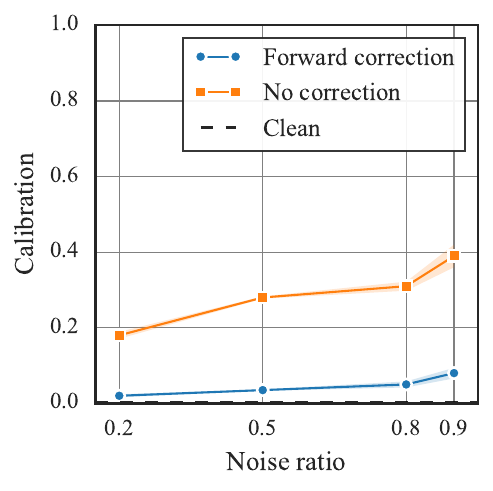}
        \caption{CIFAR-100 ECE}
        \label{fig:cifar100_self_ece}
    \end{subfigure}
    \caption{ACC and ECE comparison on CIFAR datasets under Ideal Fitted Case.}
    \label{fig:cifar_self}
\end{figure*}

\subsection{Fundamental Analysis: Information Cost of Label Noise}
\label{sec:information_cost}

Our analyses in \Cref{subsec:macroscopic_analysis} and \Cref{subsec:gradient_derivations} have dissected the specific failure modes of forward correction methods. However, these analyses focus on the properties of the correction \textit{estimators} only. We now step back to consider a more fundamental question: how does the noise process $T(x)$ impact the \textit{data itself}? 

This section adopts an information-theoretic perspective to demonstrate that the noise channel $T(x)$ inflicts a \textit{fundamental and universal penalty} by degrading the information content available per sample. This inherent information loss precedes any correction attempt and sets a baseline level of statistical difficulty for \textit{all} LNL strategies.

\subsubsection{Quantifying Information in a Single Clean Sample}
\label{subsec:info_clean}

To quantify the information a label provides about the underlying data-generating process, we adopt a standard setup to quantify the reduction in uncertainty regarding the true process. We consider a finite set of $M$ candidate hypotheses about the true clean posterior, denoted $\{\eta_m(x)\}_{m=1}^{M}$, drawn from the overall hypothesis space ($\mathcal{F}$). We assign a prior probability $\pi_m$ to each hypothesis $m$. For a fixed input $x$, the expected posterior under this prior is the mixture $\bar\eta(x) \triangleq \sum_{m=1}^{M}\pi_m\,\eta_m(x)$.

The information that a single clean label $Y$ at input $x$ reveals about the true underlying hypothesis $M$ is precisely the conditional mutual information $I(M;Y\mid X=x)$:
\begin{equation*}
\label{eq:js_identity_point}
I_{\mathrm{clean}}(x) \triangleq I(M;Y\mid X=x)
\end{equation*}
$I_{\mathrm{clean}}(x)$ therefore represents the \textit{maximum statistical signal} available from a clean label at instance $x$ for reducing our uncertainty across the hypothesis space.

\subsubsection{Information Contraction via the Noise Channel}
\label{subsec:info_contraction}

In the LNL setting, we observe a noisy label $Y^n$ instead of $Y$. The generation process follows the Markov chain:
$$
M\ \to\ (X,Y)\ \to\ (X,Y^n)
$$
The standard Data Processing Inequality (DPI) guarantees that information about $M$ cannot increase along this chain. Let $q_m(x) = T(x)^{\top}\eta_m(x)$ be the noisy posterior corresponding to hypothesis $m$. The conditional information obtainable from a noisy label is:
\[
I_{\mathrm{noisy}}(x) \triangleq I(M;Y^n\mid X=x) 
\]

\begin{theorem}[Fundamental Information Cost of Label Noise]
\label{thm:information_cost}
For any fixed input $X=x$, the information content available from the noisy label $Y^n$ is strictly less than or equal to that of the clean label $Y$:
\[
I_{\mathrm{noisy}}(x) \le I_{\mathrm{clean}}(x).
\]
If the noise channel $T(x)$ is non-trivial (neither Identify nor Permutation), the information contraction is strict, $I_{\mathrm{noisy}}(x) < I_{\mathrm{clean}}(x)$. Detailed proof in Appendix \ref{app:proof_information_cost}.
\end{theorem}

\paragraph{Synthesizing Information Cost and Optimization Failure.}
This information-theoretic analysis completes our theoretical framework. 
This \textit{information poverty} provides the final explanation for the LNL paradox. The model overfits not just because the loss function allows it, but because the data itself lacks enough information required to guide the optimization toward the correct, ideal solution $P(Y|X)$.

\begin{figure*}[ht]
    \centering
    \begin{subfigure}{0.49\textwidth}
        \centering
        \includegraphics[width=\linewidth]{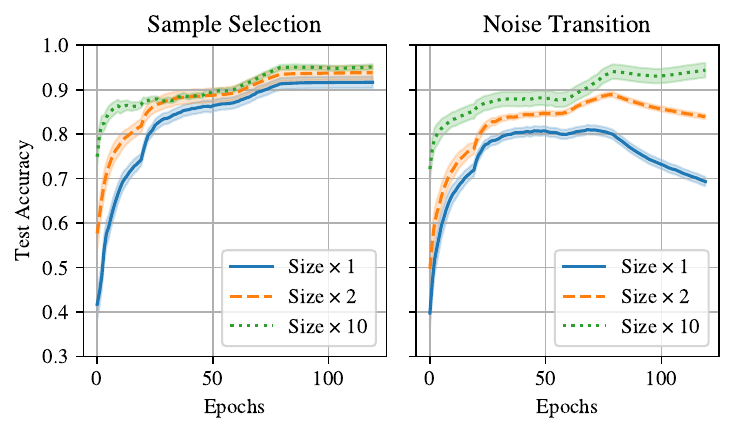}
        \caption{Accuracy}
        \label{fig:cifar10_acc_multiplied}
    \end{subfigure}
    \begin{subfigure}{0.49\textwidth}
        \centering
        \includegraphics[width=\linewidth]{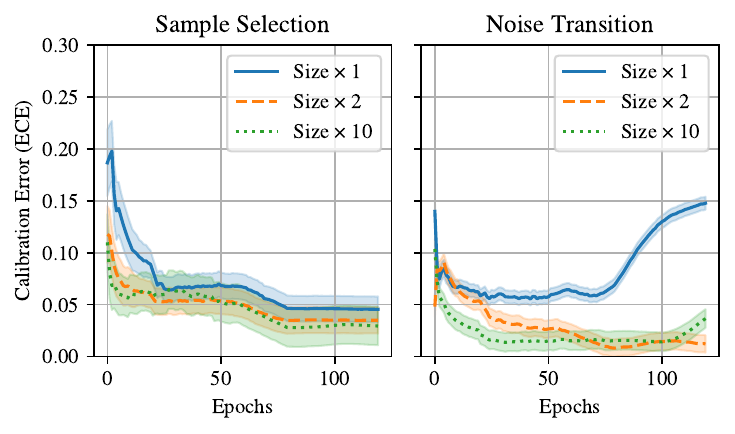}
        \caption{ECE}
        \label{fig:cifar10_ece_multiplied}
    \end{subfigure}

    \caption{Comparison of Accuracy and ECE for CIFAR-10 on multi-labeled dataset.}
    \label{fig:cifar_comparison}
\end{figure*}

\begin{table*}[tbp]
\centering
\begin{minipage}{0.6\textwidth}
\centering
\resizebox{\linewidth}{!}{
\begin{tabular}{lccccccccc}
\toprule
\textbf{Dataset} & \multicolumn{5}{c}{\textbf{CIFAR-10}} & \multicolumn{4}{c}{\textbf{CIFAR-100}} \\ \cmidrule(lr){2-6} \cmidrule(lr){7-10}
\textbf{Noise Type} & \multicolumn{4}{c}{Symmetric} & Asymmetric & \multicolumn{4}{c}{Symmetric} \\ \cmidrule(lr){2-5} \cmidrule(lr){6-6}\cmidrule(lr){7-10}
\textbf{Noise Ratio} & 20\% & 50\% & 80\% & 90\% & 40\% & 20\% & 50\% & 80\% & 90\% \\ \midrule
CE & 86.8 & 79.4 & 62.9 & 42.7 & 85.0 & 62.0 & 46.7 & 19.9 & 10.1 \\ \midrule
\rowcolor{lightgray!20}Coteaching~\citep{coteaching+} & 89.5 & 85.7 & 67.4 & 47.9 & - & 65.6 & 51.8 & 27.9 & 13.7 \\
\rowcolor{lightgray!20}PENCIL~\citep{pencil} & 92.4 & 89.1 & 77.5 & 58.9 & 88.5 & 69.4 & 57.5 & 31.1 & 15.3 \\
\rowcolor{lightgray!20}LossModel~\citep{lossmodellingbmm} & 94.0 & 92.0 & 86.8 & 69.1 & 87.4 & 73.9 & 66.1 & 48.2 & 24.3 \\
\rowcolor{lightgray!20}DivideMix~\citep{dividemix} & 96.1 & 94.6 & 93.2 & 76.0 & 93.4 & 77.3 & 74.6 & 60.2 & 31.5 \\
\midrule
Mixup~\citep{Mixup} & 95.6 &87.1 &71.6 &52.2 &67.8 &57.3 &30.8 &14.6\\
Forward~\citep{loss_correction} & 86.8 & 79.8 & 63.3 & 42.9 & 87.2 & 61.5 & 46.6 & 19.9 & 10.2 \\
FEC (Ours) & 88.1 & 87.3 & 85.6 & 82.5 & 80.2 & 60.5 & 58.6 & 52.7 & 44.1 \\
JEC (Ours) & 95.6& 88.8& 78.5& 68.5& 93.1& 73.1& 64.9& 50.1& 16.2\\ \bottomrule
\end{tabular}
}
\subcaption{Testing accuracy on CIFAR-10 and CIFAR-100 datasets. }
\label{tab:results}
\end{minipage}
\begin{minipage}{0.25\textwidth}
\centering
\resizebox{\linewidth}{!}{
\begin{tabular}{lc}
\toprule
\textbf{Method} & \textbf{Accuracy (\%)} \\ \midrule
CE & 69.03 \\
\rowcolor{lightgray!20}DivideMix~\citep{dividemix} & 74.76 \\
T-Rev~\cite{xia2019anchor} & 71.01 \\
DualT~\cite{yao2020dual} & 71.49 \\
Mixup~\citep{Mixup} & 71.29 \\
Forward~\citep{loss_correction} & 69.84 \\
FEC (Ours) & 61.85 \\
JEC (Ours) & 72.24 \\ \bottomrule
\end{tabular}
}
\subcaption{Testing accuracy on Clothing1M.}
\label{tab:clothing1m}
\end{minipage}
\caption{Experimental results across datasets.}
\label{tab:combined_results}
\end{table*}

\section{Experiments}
\label{sec:experiments}

This section moves from theoretical diagnosis to empirical validation, structured around three progressive objectives. We first confirm the superior performance of FC in the ideal state. Second, we investigate several methods to promote FC into its optimal generalization state. Finally, motivated by our information-theoretic findings, we conduct scaling experiments to verify the role of per-sample information in mitigating memorization failure. All specific experimental details are provided in Appendix \ref{supp:experiments_details}.

\subsection{Validating Ideal State}
\label{subsec:exp_paradox}

We first validate the theoretical benefits of noise correction in the Ideal Fitted Case. To \textit{approximate} the ideal functional form ($f^*$), we train a linear classifier atop a pretrained encoder, which reduces model complexity while providing robust features.

As shown in \Cref{fig:cifar_self}, the accuracy improvement for FC is clear, particularly at high noise ratios. We attribute this increased accuracy gap to the larger proportion of samples falling into the theoretically derived $\mathcal{X}_{\text{error}}$ region, making boundary correction essential. More importantly, the steady and significant improvement in calibration (ECE), compared to the accuracy gap, directly validates our theory. This confirms that the true benefits of consistency lie in posterior quality, necessitating performance evaluation that extends beyond conventional accuracy metrics.

\subsection{An Ideal-State Forcing Framework}
\label{subsec:exp_solution}

Having diagnosed the structural failure of Forward Correction (FC), we now introduce a regularized framework designed to motivate FC into its ideal state. Motivated by literature on model generalization and robust feature learning, we augment noise correction with two lightweight techniques: model pretraining and data interpolation.
Specifically, we utilize a self-supervised model pretraining and Mixup augmentations. We consider two modes:
\begin{itemize}
\item \textit{FEC (Feature-Enhanced Correction):} Train a linear classifier on top of a frozen pretrained encoder combined with Mixup and FC.
\item \textit{JEC (Joint-Enhanced Correction):} Jointly train the linear classifier with the pretrained encoder, also combined with Mixup and FC.
\end{itemize}

We evaluate these two solutions on both CIFAR and the Clothing1M datasets, with results shown in \cref{tab:results} and \cref{tab:clothing1m}. Historically, noise correction methods have often avoided direct comparison with advanced sample selection methods on mainstream benchmarks. However, we observe that by simply integrating these two lightweight strategies, our new solutions successfully push noise correction toward its ideal state, demonstrating performance that is \textit{fully competitive with} their sample selection counterparts. This confirms the great potential of noise correction methods, arguing that their empirical failure is not inherent to the approach but rather stems from a lack of effective optimization regularization.

\subsection{Noise Correction vs. Sample Selection}
\label{subsec:exp_fundamental_comparison}

In this final section, we assess the fundamental potential of noise correction methods against sample selection techniques. For a fair comparison, we operate under ideal conditions: sample selection is assumed to perfectly identify all clean samples, while noise correction is given the oracle $T$. This enables us to assess the upper theoretical limit of classification accuracy for both approaches.

To verify the role of data informativeness as suggested in \Cref{thm:information_cost}, we scale the dataset's information by scaling from single-label to multi-label samples. Results reported in \Cref{fig:cifar_comparison} confirm that as dataset information scales (from 1-label to 10-labels per sample), the ACC of noise correction method steadily improves, approaching that of ideal sample selection. Furthermore, the ECE for noise correction methods is \textit{markedly lower}, underscoring their distinct advantage in retaining predictive confidence.

\section{Conclusion}
\label{sec:conclusion}

This work rigorously resolves the long-standing paradox about statistically-consistent noise correction methods: their theoretical elegance does not translate into empirical superiority. We proves this failure is not due to estimation error, but rather a structural incompatibility between finite samples and high-capacity optimization. Through a three-pillar analysis, we found that deep networks inevitably collapse into a pathological, T-biased hard vertex. This failure mode is accelerated by a fundamental reduction in sample information caused by the noise channel.

Empirical validation confirmed our findings. We demonstrated that a simple regularization framework successfully promotes Forward Correction (FC) to its ideal generalization state, achieving performance competitive with state-of-the-art sample selection methods. This indicates a paradigm shift in LNL: moving from over-refining the noise matrix $T$ to jointly designing robust losses and optimizers. Beyond label noise, our insights into structural robustness also offer potential benefits for broader applications~\cite{Zhi2025TFAR,zeng2025FSDrive,Feng2022adaptive,Feng2025PROSAC,Feng2026NoisyValid,Sun2024LAFS,yang2023re,yang2024backdoor,Ge2025MM,luo2024arena,luo2025agentmath}.

    \noindent\textbf{Acknowledgements.} Georgios Tzimiropoulos’s work was supported by UK Research and Innovation (UKRI) under the UK government’s Horizon Europe funding guarantee (Grant No. 10099264) and by the European Union under Horizon Europe Grant Agreement No. 101135800 (RAIDO).
{
    \small
    \bibliographystyle{ieeenat_fullname}
    \bibliography{main}
}

\clearpage
\setcounter{page}{1}
\maketitlesupplementary

\appendix
\setcounter{footnote}{0}   


\appendix

\section{Proofs for the Ideal Fitted Case (Theorem \ref{theorem:ideal_performance})}
\label{app:proof_ideal_case}

In this appendix, we provide the rigorous derivation for the performance of Forward Correction (FC) and No Correction (NC) in the ideal asymptotic limit ($N \to \infty$). We rely on the notation defined in the main text: $\eta(x) \triangleq P(Y|X=x)$ denotes the clean posterior, and $\eta^n(x) \triangleq P(Y^n|X=x)$ denotes the noisy posterior. Recall the definition of inherent uncertainty: $\delta(x) \triangleq 1 - \max_k \eta_k(x)$.

\subsection{Proof of Theorem \ref{theorem:ideal_performance}(a): Accuracy Analysis}
\label{app:proof_acc_ideal}

\begin{proof}
We analyze the expected accuracy for both methods and derive the performance gap $\Delta$.

\paragraph{1. Forward Correction (FC).}
Under the ideal fitted assumption, FC is statistically consistent. It successfully recovers the clean posterior, i.e., $\hat{p}_{\fc}(x) = \eta(x)$. Consequently, the model's prediction $Y^f_{\fc}(x)$ coincides with the Clean Bayes-Optimal classifier $Y^*(x) = \arg\max_k \eta_k(x)$.
The accuracy is derived as:
\begin{align}
\mathrm{ACC}(f_{\fc}) &= \mathbb{E}_{(X,Y)} \left[ \mathbb{I}(Y=Y^f_{\fc}(X)) \right] \nonumber \\
&= \mathbb{E}_X \left[ \sum_{y \in \mathcal{Y}} P(Y=y|X) \cdot \mathbb{I}(y=Y^*(X)) \right] \nonumber \\
&= \mathbb{E}_X \left[ \eta_{Y^*(X)}(X) \right] \nonumber \\
&= \mathbb{E}_X \left[ 1-\delta(X) \right] = 1 - \mathbb{E}_X[\delta(X)].
\label{eq:acc_fc_final}
\end{align}
This confirms the first claim of Theorem \ref{theorem:ideal_performance}(a).

\paragraph{2. No Correction (NC).}
The NC objective minimizes the risk with respect to the noisy labels. The population minimizer yields the noisy posterior $\hat{p}_{\nc}(x) = \eta^n(x)$. Thus, the prediction follows the Noisy Bayes-Optimal classifier $\tilde{Y}^*(x) = \arg\max_k \eta^n_k(x)$.
Using the partition from Definition \ref{def:consistency_partition}, we decompose the accuracy via the Law of Total Expectation:
\begin{align}
\mathrm{ACC}(f_\nc) &= \mathbb{E}_X \left[ P(Y=\tilde{Y}^*(X) \mid X) \right] \nonumber \\
&= P(\mathcal{X}_{\text{correct}}) \cdot \mathbb{E}_{X|\mathcal{X}_{\text{correct}}} \left[ \eta_{\tilde{Y}^*(X)}(X) \right] 
+ P(\mathcal{X}_{\text{error}}) \cdot \mathbb{E}_{X|\mathcal{X}_{\text{error}}} \left[ \eta_{\tilde{Y}^*(X)}(X) \right].
\label{eq:acc_nc_split}
\end{align}
We analyze the term $\eta_{\tilde{Y}^*(X)}(X)$ in each regime:
\begin{itemize}
    \item \textbf{Regime 1 ($\mathcal{X}_{\text{correct}}$):} By definition, $\tilde{Y}^*(x) = Y^*(x)$. Thus, $\eta_{\tilde{Y}^*(x)}(x) = \eta_{Y^*(x)}(x) = 1-\delta(x)$.
    \item \textbf{Regime 2 ($\mathcal{X}_{\text{error}}$):} By definition, the predicted class is strictly suboptimal, i.e., $\tilde{Y}^*(x) = k_{err}$ where $k_{err} \neq Y^*(x)$.
\end{itemize}

\paragraph{3. The Accuracy Gap ($\Delta$).}
The gap is defined as $\Delta \triangleq \mathrm{ACC}(f_\fc) - \mathrm{ACC}(f_\nc)$. Decomposing $\mathrm{ACC}(f_\fc)$ similarly over the partition, the terms on $\mathcal{X}_{\text{correct}}$ are identical and cancel out. We are left with the difference on the error set:
\begin{equation}
\Delta = P(\mathcal{X}_{\text{error}}) \cdot \mathbb{E}_{X|\mathcal{X}_{\text{error}}} \left[ \eta_{Y^*(X)}(X) - \eta_{\tilde{Y}^*(X)}(X) \right].
\end{equation}
For any $x \in \mathcal{X}_{\text{error}}$, let $k_{err} = \tilde{Y}^*(x)$. The probability mass of this erroneous class, $\eta_{k_{err}}(x)$, is bounded by two constraints:
\begin{enumerate}
    \item It is sub-optimal: $\eta_{k_{err}}(x) \le \eta_{Y^*(x)}(x) = 1-\delta(x)$.
    \item It is part of the residual mass: $\eta_{k_{err}}(x) \le \sum_{k \neq Y^*} \eta_k(x) = \delta(x)$.
\end{enumerate}
Thus, $\eta_{k_{err}}(x) \le \min(\delta(x), 1-\delta(x))$. Substituting this into the gap equation:
\begin{align}
\eta_{Y^*(x)}(x) - \eta_{k_{err}}(x) &\ge (1-\delta(x)) - \min(\delta(x), 1-\delta(x)) \nonumber \\
&= \max(0, 1-2\delta(x)).
\end{align}
Therefore, the lower bound for the gap is:
\begin{equation}
\Delta \ge P(\mathcal{X}_{\text{error}}) \cdot \mathbb{E}_{X \mid \mathcal{X}_{\text{error}}}\Big[ \max\big(0, 1-2\delta(X)\big) \Big] \ge 0.
\end{equation}
This proves the non-negative gap and completes the proof for Part (a).
\end{proof}

\subsection{Proof of Theorem \ref{theorem:ideal_performance}(b): ECE Analysis}
\label{app:proof_ece_ideal}

\begin{proof}
Let $C(X) = \max_k \hat{p}_k(X)$ be the prediction confidence and $Y^f(X) = \arg\max_k \hat{p}_k(X)$ be the predicted label. The Expected Calibration Error (ECE) is defined as:
\begin{equation}
\mathrm{ECE}(f) = \mathbb{E}_C \left[ \left| P(Y=Y^f(X) \mid C(X)=C) - C \right| \right].
\end{equation}
By defining the \textit{per-sample calibration gap} as $\Delta_{\text{cal}}(x) \triangleq | P(Y=Y^f(x)|x) - \hat{p}_{Y^f(x)}(x) |$, we can rewrite the ECE as the expected local calibration error:
\begin{equation}
\mathrm{ECE}(f) = \mathbb{E}_X \left[ \Delta_{\text{cal}}(X) \right].
\end{equation}

\paragraph{1. Forward Correction (FC).}
In the ideal case, $\hat{p}_{\fc}(x) = \eta(x)$.
The confidence is $C(x) = \eta_{Y^*(x)}(x)$. The true accuracy of this prediction is $P(Y=Y^*(x)|x) = \eta_{Y^*(x)}(x)$.
Since the model output matches the ground truth posterior, the gap vanishes:
\begin{equation}
\Delta_{\text{cal}}^{\fc}(x) = | \eta_{Y^*(x)}(x) - \eta_{Y^*(x)}(x) | = 0 \implies \mathrm{ECE}(f_{\fc}) = 0.
\end{equation}

\paragraph{2. No Correction (NC).}
In the ideal case, $\hat{p}_{\nc}(x) = \eta^n(x)$.
The model's confidence is derived from the noisy posterior: $C(x) = \eta^n_{\tilde{Y}^*(x)}(x)$.
However, the true correctness of the prediction depends on the clean posterior: $P(Y=\tilde{Y}^*(x)|x) = \eta_{\tilde{Y}^*(x)}(x)$.
The calibration gap is:
\begin{equation}
\Delta_{\text{cal}}^{\nc}(x) = \left| \eta_{\tilde{Y}^*(x)}(x) - \eta^n_{\tilde{Y}^*(x)}(x) \right|.
\end{equation}
Unless the transition matrix $T(x)$ is the identity (no noise) or a specific permutation that preserves diagonal dominance exactly, generally $\eta(x) \neq \eta^n(x)$. Thus, $\Delta_{\text{cal}}^{\nc}(x) > 0$ for some $x$, implying $\mathrm{ECE}(f_{\nc}) > 0$.
\end{proof}

\section{Derivation of the Empirical Minimizer for FC (Eq. \ref{eq:fc_collapse})}
\label{app:proof_fc_collapse}

In this section, we provide the formal derivation for the optimal solution $\hat{p}(x)$ that minimizes the single-sample Forward Correction (FC) loss. 

\paragraph{1. Optimization Problem.}
We analyze the failure mode where a high-capacity network memorizes the training set, driving the empirical risk $\widehat{R}(f) \to 0$. This implies minimizing the loss for each sample $(x, y^n)$ independently.

The single-sample FC loss for a prediction $\hat{p}(x)$ is:
\begin{equation}
    \ell_{\fc}(\hat{p}(x) \mid x, y^n) = -\log \left( \sum_{k=1}^K T_{k, y^n}(x) \cdot \hat{p}_k(x) \right).
\end{equation}
Our objective is to find the optimal probability vector $\hat{p}^*$ that minimizes this loss, subject to the constraint that $\hat{p}^*$ lies on the probability simplex $\Delta^{K-1}$:
\begin{equation}
\label{eq:app_optim_problem}
    \hat{p}^* = \arg\min_{\hat{p} \in \Delta^{K-1}} \ell_{\fc}(\hat{p} \mid x, y^n).
\end{equation}

\paragraph{2. Equivalent Linear Objective.}
The function $g(z) = -\log(z)$ is strictly monotonically decreasing for $z > 0$. Therefore, minimizing $\ell_{\fc}$ is mathematically equivalent to maximizing its argument:
\begin{equation}
\label{eq:app_optim_simplified}
    \hat{p}^* = \arg\max_{\hat{p} \in \Delta^{K-1}} \left( \sum_{k=1}^K T_{k, y^n}(x) \cdot \hat{p}_k \right).
\end{equation}

\paragraph{3. Solution via Linear Programming.}
For a fixed sample $(x, y^n)$, the transition probabilities $T_{k, y^n}(x)$ (which form the $y^n$-th column of $T(x)$) are constants. Let us define a constant vector $\mathbf{c} \in \mathbb{R}^K$ where each element $c_k = T_{k, y^n}(x)$.

The optimization problem simplifies to:
\begin{equation}
    \arg\max_{\hat{p} \in \Delta^{K-1}} \left( \mathbf{c}^\top \hat{p} \right).
\end{equation}
This is a \textbf{Linear Program (LP)}: we are maximizing a linear objective function ($\mathbf{c}^\top \hat{p}$) over a convex polytope (the probability simplex $\Delta^{K-1}$).

\paragraph{4. Optimal Solution at Vertex.}
A fundamental theorem of linear programming states that the maximum of a linear function over a convex polytope must be achieved at one of the polytope's vertices.
The vertices of the probability simplex $\Delta^{K-1}$ are the set of standard basis vectors (one-hot vectors):
\[
    \mathcal{V} = \{ \mathbf{e}_1, \mathbf{e}_2, \dots, \mathbf{e}_K \}.
\]
To find the optimal solution, we evaluate the objective function at each vertex $\mathbf{e}_j$:
\[
    \mathbf{c}^\top \mathbf{e}_j = \sum_{k=1}^K c_k \cdot (\mathbf{e}_j)_k = c_j = T_{j, y^n}(x).
\]
The objective is maximized by choosing the vertex $\mathbf{e}_{k^*}$ that corresponds to the largest coefficient $c_{k^*}$. We define this optimal index $k^*$ as:
\[
    k_{\fc}^*(x) \triangleq \arg\max_{k \in \{1, \dots, K\}} T_{k, y^n}(x).
\]
Therefore, the unique optimal solution that minimizes the single-sample FC loss is:
\[
    \hat{p}^* = \mathbf{e}_{k_{\fc}^*(x)}.
\]
This completes the derivation of \Cref{eq:fc_collapse}.

\paragraph{Implications.}
This derivation formally proves that the empirical minimizer of the FC loss is \textit{not} the ``correct" soft posterior $\eta(x)$. Instead, the loss function creates an optimization landscape whose global minimum (for a finite sample) is a hard, one-hot vector. 

\section{Proofs for the Empirical Overfitted Case (Theorem \ref{theorem:macroscopic_tradeoff})}
\label{supp:proof_acc_ece_fitted}

In this appendix, we provide the detailed derivations for the accuracy and calibration properties in the ``Empirical Overfitted Case" (\Cref{subsubsec:overfit_case}). We assume the model has memorized the training data, collapsing its predictions to one-hot vectors as derived in \Cref{app:proof_fc_collapse}.

We use the notation $\eta_k(x) \triangleq P(Y=k \mid X=x)$ for the clean posterior and $\delta(x) \triangleq 1 - \eta_{Y^*}(x)$ for the inherent uncertainty. The accuracy is $\mathrm{ACC} = \mathbb{E}_X [ P(Y^f = Y \mid X) ]$.

\subsection{Proof of Theorem \ref{theorem:macroscopic_tradeoff}(a): Accuracy (ACC)}

We first derive the exact conditional accuracy $\mathrm{ACC}(x) = P(Y^f = Y \mid X=x)$ for both methods.

\paragraph{No Correction (NC)}
The NC solution memorizes the noisy label, so $Y^f = Y^n$. The conditional accuracy is the probability that the observed noisy label matches the (unseen) true clean label.
\begin{align}
    \mathrm{ACC}_{\nc}(x) &= P(Y^f = Y \mid x) = P(Y^n = Y \mid x) \nonumber \\
    &= \sum_{k \in \mathcal{Y}} P(Y^n = Y, Y=k \mid x) \nonumber \\
    &= \sum_{k \in \mathcal{Y}} P(Y^n = k, Y=k \mid x) \quad \text{(since } Y^n=Y \text{ implies } Y^n=k \text{ and } Y=k \text{)} \nonumber \\
    &= \sum_{k=1}^K P(Y^n=k \mid Y=k, x) P(Y=k \mid x) \nonumber \\
    &= \sum_{k=1}^K T_{k,k}(x) \eta_k(x) \label{eq:app_nc_exact_derived}
\end{align}
To analyze its dependence on $\delta(x)$, we split the sum into the dominant class ($k=Y^*$) and the residual $\mathcal{R}_{\nc}(x)$:
\begin{align}
    \mathrm{ACC}_{\nc}(x) &= \eta_{Y^*}(x) T_{Y^*,Y^*}(x) + \sum_{k \ne Y^*} \eta_k(x) T_{k,k}(x) \nonumber \\
    &= (1 - \delta(x)) T_{Y^*,Y^*}(x) + \mathcal{R}_{\nc}(x) \label{eq:app_nc_delta}
\end{align}
Since $0 \le T_{k,k}(x) \le 1$ and $\sum_{k \ne Y^*} \eta_k(x) = \delta(x)$, the residual is bounded: $0 \le \mathcal{R}_{\nc}(x) \le \delta(x)$. Thus, $\mathcal{R}_{\nc}(x) = \mathcal{O}(\delta(x))$. The dominant first-order approximation is:
\begin{equation}
    \mathrm{ACC}_{\nc}(x) \approx (1 - \delta(x)) T_{Y^*,Y^*}(x)
\end{equation}

\paragraph{Forward Correction (FC)}
From \Cref{app:proof_fc_collapse}, the FC prediction is $Y^f = k^*(Y^n)$, where $k^*(j) \triangleq \arg\max_k T_{k,j}(x)$. The conditional accuracy is the probability that this prediction matches the true label $Y$.
\begin{align}
    \mathrm{ACC}_{\fc}(x) &= P(Y^f = Y \mid x) = P(k^*(Y^n) = Y \mid x) \nonumber \\
    &= \sum_{k \in \mathcal{Y}} P(k^*(Y^n) = Y, Y=k \mid x) \nonumber \\
    &= \sum_{k \in \mathcal{Y}} P(Y=k \mid x) P(k^*(Y^n)=k \mid Y=k, x) \quad \text{(as } Y^n \text{ depends on } Y \text{)} \nonumber \\
    &= \sum_{k=1}^K \eta_k(x) \left( \sum_{j : k^*(j)=k} P(Y^n=j \mid Y=k, x) \right) \nonumber \\
    &= \sum_{k=1}^K \eta_k(x) \underbrace{\left( \sum_{j : k^*(j)=k} T_{k,j}(x) \right)}_{\triangleq C_k(x)} \label{eq:app_fc_exact_derived}
\end{align}
Let $C_k(x)$ be the sum of probabilities of transitions from class $k$ to any noisy label $j$ that maps back to $k$. The exact accuracy is $\mathrm{ACC}_{\fc}(x) = \sum_k \eta_k(x) C_k(x)$. We split this sum:
\begin{align}
    \mathrm{ACC}_{\fc}(x) &= \eta_{Y^*}(x) C_{Y^*}(x) + \sum_{k \ne Y^*} \eta_k(x) C_k(x) \nonumber \\
    &= (1 - \delta(x)) C_{Y^*}(x) + \mathcal{R}_{\fc}(x) \label{eq:app_fc_delta}
\end{align}
Since $C_k(x) = \sum_{j: \dots} T_{k,j}(x) \le \sum_j T_{k,j}(x) = 1$, the residual $\mathcal{R}_{\fc}(x)$ is also $\mathcal{O}(\delta(x))$. The first-order approximation is:
\begin{align}
    \mathrm{ACC}_{\fc}(x) \approx (1 - \delta(x)) C_{Y^*}(x) \label{eq:app_fc_approx}
\end{align}

\paragraph{Comparison: The Gain/Loss Trade-off}
The relative performance is driven by the difference between $C_{Y^*}(x)$ and $T_{Y^*,Y^*}(x)$. The first-order accuracy gap is:
\begin{equation}
    \Delta_{\mathrm{ACC}}(x) \approx (1 - \delta(x)) \left[ C_{Y^*}(x) - T_{Y^*,Y^*}(x) \right]
\end{equation}
To understand this gap, we decompose $C_{Y^*}(x)$ by splitting its sum at $j=Y^*$:
\begin{align}
    C_{Y^*}(x) &= \sum_{j: k^*(j)=Y^*} T_{Y^*,j}(x) \nonumber \\
    &= T_{Y^*,Y^*}(x) \cdot \mathbb{I}(k^*(Y^*) = Y^*) + \sum_{j \ne Y^*} T_{Y^*,j}(x) \cdot \mathbb{I}(k^*(j) = Y^*)
\end{align}
Substituting this into the gap equation $\left[ C_{Y^*}(x) - T_{Y^*,Y^*}(x) \right]$ yields:
\begin{align}
    &= \left[ T_{Y^*,Y^*}(x) \mathbb{I}(k^*(Y^*) = Y^*) + \sum_{j \ne Y^*} \dots \right] - T_{Y^*,Y^*}(x) \nonumber \\
    &= \sum_{j \ne Y^*} T_{Y^*,j}(x) \mathbb{I}(k^*(j) = Y^*) + T_{Y^*,Y^*}(x) \left( \mathbb{I}(k^*(Y^*) = Y^*) - 1 \right) \nonumber \\
    &= \underbrace{\sum_{j \ne Y^*} T_{Y^*,j}(x) \mathbb{I}(k^*(j) = Y^*)}_{\text{Gain: Errors corrected by FC}} - \underbrace{T_{Y^*,Y^*}(x) \mathbb{I}(k^*(Y^*) \ne Y^*)}_{\text{Loss: Correct labels mis-corrected by FC}} \label{eq:app_tradeoff}
\end{align}
This confirms that FC's performance is a fragile trade-off: it gains accuracy only if it correctly maps erroneous labels (like $j \ne Y^*$) back to $Y^*$, but it loses accuracy if it maps the \textit{correct} label $Y^*$ to something else ($k^*(Y^*) \ne Y^*$).

\paragraph{Intuition and Visualization of the Gain/Loss Trade-off.}
To build a strong intuition for the trade-off just derived in \Cref{eq:app_tradeoff}, we provide a concrete visualization of the collapsed solution $k^*(j) \triangleq \arg\max_{k} T_{k,j}(x)$. This is a deterministic lookup table found by scanning each \textbf{column $j$} of the $T$-matrix to find the \textbf{row $k$} with the maximum value.

Consider this 3-class ``pathological" but valid T-matrix (rows sum to 1, column max bolded):
\[
T(x) = \bordermatrix{
 & Y^n=\text{Bird (j=1)} & Y^n=\text{Dog (j=2)} & Y^n=\text{Cat (j=3)} \cr
Y=\text{Bird (k=1)} & \mathbf{0.5} & \mathbf{0.45} & 0.05 \cr
Y=\text{Dog (k=2)} & 0.25& 0.4 & \mathbf{0.35} \cr
Y=\text{Cat (k=3)} & 0.33 & 0.33 & 0.34
}
\]
The lookup table $k^*(j)$ becomes (based on the column-wise bolded maximums):
\begin{itemize}
    \item $k^*(1) = \arg\max(0.5, 0.25, 0.33) = 1$ (Observed `Bird' $\to$ Predict `Bird')
    \item $k^*(2) = \arg\max(0.45, 0.4, 0.33) = 1$ (Observed `Dog' $\to$ Predict `Bird')
    \item $k^*(3) = \arg\max(0.05, 0.35, 0.34) = 2$ (Observed `Cat' $\to$ Predict `Dog')
\end{itemize}

This pathological mapping $j \to k^*(j)$ directly explains the Gain/Loss terms from \Cref{eq:app_tradeoff}:

\begin{itemize}
    \item \textbf{Loss Term Activation ($\mathbb{I}(k^*(Y^*) \ne Y^*)$):}
    This term activates when a ``correct" noisy label is ``mis-corrected". For example, if the true label is $Y^*=2$ (`Dog') and the noisy label is also $Y^n=2$ (a correct pair), the NC baseline would be correct. However, the FC model uses its lookup table, finds $k^*(2)=1$, and predicts `Bird', thus \textbf{creating a loss}. The same occurs for $Y^*=3$ (`Cat'), where the correct label $Y^n=3$ is mapped to $k^*(3)=2$.
    
    \item \textbf{Gain Term Activation ($\mathbb{I}(k^*(j) = Y^*)$ for $j \ne Y^*$):}
    This term activates when an ``incorrect" noisy label is ``corrected". For example, if the true label is $Y^*=1$ (`Bird') but the noisy label is $Y^n=2$ (`Dog') (an error pair), the NC baseline would be wrong. However, the FC model finds $k^*(2)=1$, which matches $Y^*=1$, thus \textbf{creating a gain}. The same occurs for $Y^*=2$ (`Dog') when $Y^n=3$ (`Cat'), as $k^*(3)=2$.
\end{itemize}
Therefore, the performance of the overfitted FC model depends on the fragile balance between these gain and loss terms, which are dictated entirely by the T-matrix structure.

\paragraph{Symmetric Noise Case.}
Now we apply this framework to the symmetric noise case. With rate $\rho < (K-1)/K$, the T-matrix is constant, and $T_{j,j} = 1-\rho > T_{i,j} = \rho/(K-1)$ for $i \ne j$.
This means the maximum of every column $j$ is strictly on the diagonal. Therefore:
\[
    k^*(j) = j \quad \text{for all } j.
\]
We plug this into the Gain/Loss equation (\Cref{eq:app_tradeoff}):
\begin{itemize}
    \item \textbf{Gain Term:} $\sum_{j \ne Y^*} T_{Y^*,j} \mathbb{I}(j = Y^*) = 0$ (the sum is over an empty set).
    \item \textbf{Loss Term:} $T_{Y^*,Y^*} \mathbb{I}(Y^* \ne Y^*) = 0$.
\end{itemize}
The first-order accuracy gap is zero: $\mathrm{ACC}_{\fc}(x) \approx \mathrm{ACC}_{\nc}(x)$.
The total accuracy for both methods collapses to:
\[
\mathrm{ACC} \approx \mathbb{E}_X \left[ (1 - \delta(x)) (1-\rho) \right] = (1-\rho)(1-\mathbb{E}[\delta(X)])
\]
This analytically confirms the ``solution collapse" to the same suboptimal baseline observed in \Cref{fig:ideal_failure}.

\subsection{Proof of Theorem \ref{theorem:macroscopic_tradeoff}(b): ECE}

\begin{proof}
For both NC and FC, the overfitted solution is a one-hot vector (e.g., $\hat{p} = \mathbf{e}_{k^*}$).
The prediction confidence is therefore $C(X) = \max_k \hat{p}_k(X) = 1$ for all samples.
The ECE formula is defined as $\mathrm{ECE}(f) = \mathbb{E}_C \left[ \left| P(Y=Y^f \mid C) - C \right| \right]$.
Since $C=1$ everywhere, the expectation simplifies to a single point:
\begin{align}
    \mathrm{ECE}(f) &= \left| P(Y=Y^f \mid C=1) - 1 \right| \nonumber \\
    &= \left| P(Y=Y^f) - 1 \right| \quad \text{(since } C=1 \text{ provides no new information)} \nonumber \\
    &= \left| \mathrm{ACC}(f) - 1 \right| \nonumber \\
    &= 1 - \mathrm{ACC}(f) \quad \text{(since } \mathrm{ACC}(f) \le 1 \text{)}
\end{align}
This proves that in the overfitted, hard-label regime, ECE is perfectly and negatively coupled with accuracy.
\end{proof}

\section{Gradient Derivation for Forward Correction (FC)}
\label{app:proof_fc_gradient}

In this section, we derive the gradient of the Forward Corrected loss $\ell_{\fc}$ with respect to a single logit $f_k(x)$, as referenced in \Cref{subsec:gradient_derivations}.

The loss for a single sample $(x, y^n)$ is defined as:
\begin{equation}
    \ell_{\fc} = -\log(z_{y^n})
\end{equation}
where $z$ is the model's predicted noisy posterior vector, $z = T(x)^\top \hat{p}(x)$, and $\hat{p}(x) = \text{softmax}(f(x))$. The term $z_{y^n}$ is the single component corresponding to the observed noisy label $y^n$:
\[
z_{y^n} = \sum_{i=1}^K T_{i, y^n}(x) \hat{p}_i(x)
\]
We use the chain rule to compute $\frac{\partial \ell_{\fc}}{\partial f_k}$:
\[
    \frac{\partial \ell_{\fc}}{\partial f_k} = \frac{\partial \ell_{\fc}}{\partial z_{y^n}} \cdot \frac{\partial z_{y^n}}{\partial f_k}
\]

\paragraph{Step 1: Derivative of Loss w.r.t. $z_{y^n}$}
The derivative of the negative logarithm is:
\[
    \frac{\partial \ell_{\fc}}{\partial z_{y^n}} = -\frac{1}{z_{y^n}}
\]

\paragraph{Step 2: Derivative of $z_{y^n}$ w.r.t. logit $f_k$}
We apply the chain rule again, summing over all clean probabilities $\hat{p}_j$:
\begin{align*}
    \frac{\partial z_{y^n}}{\partial f_k} &= \sum_{j=1}^K \frac{\partial z_{y^n}}{\partial \hat{p}_j} \cdot \frac{\partial \hat{p}_j}{\partial f_k} \\
    &= \sum_{j=1}^K T_{j, y^n}(x) \cdot \frac{\partial \hat{p}_j}{\partial f_k}
\end{align*}
We use the standard softmax derivative $\frac{\partial \hat{p}_j}{\partial f_k} = \hat{p}_j (\delta_{jk} - \hat{p}_k)$, where $\delta_{jk}$ is the Kronecker delta.
\begin{align*}
    \frac{\partial z_{y^n}}{\partial f_k} &= \sum_{j=1}^K T_{j, y^n} \hat{p}_j (\delta_{jk} - \hat{p}_k) \\
    &= \underbrace{T_{k, y^n} \hat{p}_k (1 - \hat{p}_k)}_{\text{Case } j=k} + \underbrace{\sum_{j \ne k} T_{j, y^n} \hat{p}_j (-\hat{p}_k)}_{\text{Case } j \ne k} \\
    &= (T_{k, y^n} \hat{p}_k - T_{k, y^n} \hat{p}_k^2) - \hat{p}_k \sum_{j \ne k} T_{j, y^n} \hat{p}_j \\
    &= T_{k, y^n} \hat{p}_k - \hat{p}_k \left( T_{k, y^n} \hat{p}_k + \sum_{j \ne k} T_{j, y^n} \hat{p}_j \right) \\
    &= T_{k, y^n} \hat{p}_k - \hat{p}_k \left( \sum_{j=1}^K T_{j, y^n} \hat{p}_j \right) \\
    &= T_{k, y^n} \hat{p}_k - \hat{p}_k z_{y^n}
\end{align*}

\paragraph{Step 3: Combining the Terms}
We substitute the results from Step 1 and Step 2 back into the main chain rule formula:
\begin{align}
    \frac{\partial \ell_{\fc}}{\partial f_k} &= \left( -\frac{1}{z_{y^n}} \right) \cdot (T_{k, y^n} \hat{p}_k - \hat{p}_k z_{y^n}) \nonumber \\
    &= -\frac{T_{k, y^n} \hat{p}_k}{z_{y^n}} + \frac{\hat{p}_k z_{y^n}}{z_{y^n}} \nonumber \\
    &= \hat{p}_k - \frac{T_{k, y^n} \hat{p}_k}{z_{y^n}} \label{eq:app_grad_final}
\end{align}
Recalling the definition $z_{y^n} = \sum_{j} T_{j, y^n} \hat{p}_j$, we obtain the final form presented in \Cref{eq:fc_gradient}:
\begin{equation}
    \frac{\partial \ell_{\fc}}{\partial f_k(x)} = \hat{p}_k - q_k
\end{equation}
where $q_k = \frac{T_{k,y^n}(x) \hat{p}_k}{\sum_{j} T_{j,y^n}(x) \hat{p}_j}$. As noted in the main text, this $q_k$ term is precisely the model's estimated reverse posterior $P(Y=k \mid Y^n=y^n, x)$ via Bayes' rule.

This gradient form $\hat{p}_k - q_k$ should be contrasted with the standard Cross-Entropy gradient, $\hat{p}_k - \mathbb{I}\{y^n=k\}$. Instead of a hard `1' pulling the gradient for the observed class, FC uses a ``soft" target $q_k$ distributed over all classes $k$, which explains the ``gradient softening" effect discussed in the main paper.

\section{Analysis of Optimization Dynamics and Gradient Saturation}
\label{app:proof_gradient_flow}

In this appendix, we analyze the optimization dynamics of the Forward Correction (FC) loss. As shown in Appendix \ref{app:proof_fc_collapse}, the theoretical global minimum of the single-sample (overfitted) loss is not the clean label $\mathbf{e}_{y^*}$, but the 'collapsed' solution $\mathbf{e}_{k_{\fc}^*}$.
Here, we show here that the practical optimization path does not even guarantee convergence to this theoretical minimum.

Our analysis proceeds in two steps:
\begin{enumerate}
    \item \textbf{Global Gradient Flow:} We first analyze the vector field of the gradient $\nabla_{\mathbf{f}} \ell_{\fc}$. We confirm that the gradient flow, when viewed globally, does indeed point towards the correct theoretical minimum $\mathbf{e}_{k_{\fc}^*}$ (the solution from the Linear Program analysis in Appendix \ref{app:proof_fc_collapse}).
    \item \textbf{Local Gradient Saturation:} We then show that due to the Softmax parameterization, the gradient magnitude vanishes near \textit{all} simplex vertices. This creates 'dead zones' (local minima) around sub-optimal attractors, most notably the noisy label vertex $\mathbf{e}_{y^n}$.
\end{enumerate}
This analysis demonstrates that while the loss function's \textit{global} landscape points to $\mathbf{e}_{k_{\fc}^*}$, its \textit{local} properties trap the optimizer. The final `pseudo-converged' solution is therefore path-dependent, not guaranteed to be the theoretical optimum, and often defaults to the noisy vertex $\mathbf{e}_{y^n}$ due to early-learning dynamics.

\paragraph{Theoretical Optimum and Global Gradient Flow}
\label{app:sub_grad_flow}

As derived in Appendix \ref{app:proof_fc_collapse}, the FC loss for a sample $(x, y^n)$ is:
\begin{equation}
    \ell_{\fc}(\hat{\mathbf{p}}) = - \log \left( \sum_{k=1}^K \hat{p}_k T_{k, y^n} \right)
\end{equation}
Minimizing this is equivalent to the Linear Program $\max_{\hat{\mathbf{p}} \in \Delta^{K-1}} \sum_{k=1}^K \hat{p}_k T_{k, y^n}$. The global optimum $\hat{\mathbf{p}}^*$ is the one-hot vector $\mathbf{e}_{k_{\fc}^*}$, where $k_{\fc}^* = \arg\max_{k} T_{k, y^n}$.

An analysis of the gradient flow (as visualized on a 3-class simplex example in \cref{fig:gradient}) confirms this. The vector field of the gradient $\nabla_{\mathbf{f}} \ell_{\fc}$ globally points away from all other vertices and towards the single vertex $\mathbf{e}_{k_{\fc}^*}$. Theoretically, a gradient descent algorithm with perfect information should converge to $k_{\fc}^*$.

\paragraph{Gradient Saturation and Pseudo-Convergence}
\label{app:sub_grad_saturation}

The practical failure arises from the Softmax parameterization $\hat{p}_k = \text{softmax}(f_k)$. The gradient of the loss with respect to the logits $\mathbf{f}$ is:
\begin{equation}
    \frac{\partial \ell}{\partial f_k} = \sum_{j=1}^K \frac{\partial \ell}{\partial \hat{p}_j} \frac{\partial \hat{p}_j}{\partial f_k}
\end{equation}
The Softmax derivative $\frac{\partial \hat{p}_j}{\partial f_k} = \hat{p}_j (\delta_{jk} - \hat{p}_k)$ approaches 0 as $\hat{\mathbf{p}}$ approaches \textit{any} vertex $\mathbf{e}_i$.
Consequently, the magnitude of the logit gradient vanishes near all vertices:
\begin{equation}
    \lim_{\hat{\mathbf{p}} \rightarrow \mathbf{e}_i} \| \nabla_{\mathbf{f}} \ell_{\fc} \| \approx 0 \quad \text{for any } i \in \{1, \dots, K\}
\end{equation}
This phenomenon creates ``gradient plateaus" or ``dead zones" around \textit{every} vertex.

\paragraph{Dynamics Analysis: Trapped in a Local Minimum}
\label{app:sub_grad_dynamic}

The existence of these ``dead zones" means the final convergence point is path-dependent. In Noisy Label Learning, models famously exhibit an ``early learning" phase (fitting simple patterns) followed by a ``memorization" phase (fitting noisy labels).
\begin{enumerate}
    \item The model quickly learns to fit the dominant noisy labels, causing its prediction $\hat{\mathbf{p}}$ to approach the noisy vertex $\mathbf{e}_{y^n}$.
    \item As $\hat{\mathbf{p}} \to \mathbf{e}_{y^n}$, the model enters the gradient saturation ``dead zone" of this vertex.
    \item Although $\mathbf{e}_{y^n}$ is not the global minimum of $\ell_{\fc}$ (assuming $k_{\fc}^* \neq y^n$), the gradient pointing away from $\mathbf{e}_{y^n}$ and towards the true minimum $\mathbf{e}_{k_{\fc}^*}$ becomes infinitesimally small.
    \item SGD, with limited step size, fails to escape this local basin of attraction.
\end{enumerate}
This results in \textit{pseudo-convergence}: the model remains ``trapped" at the wrong label $\mathbf{e}_{y^n}$. Therefore, the actual overfitted solution observed in practice is not the theoretical optimum $\mathbf{e}_{k_{\fc}^*}$, but rather a sub-optimal local minimum $\mathbf{e}_{y^n}$ that acts as a strong attractor.

\section{Proof of \Cref{thm:information_cost}: Fundamental Information Cost}
\label{app:proof_information_cost}

We provide the rigorous proof for Theorem \ref{thm:information_cost}. We restate the definitions for a fixed input $X=x$:
\begin{itemize}
    \item $I_{\mathrm{clean}}(x) \triangleq I(M; Y \mid X=x)$
    \item $I_{\mathrm{noisy}}(x) \triangleq I(M; Y^n \mid X=x)$
\end{itemize}
Our goal is to prove that $I_{\mathrm{noisy}}(x) \le I_{\mathrm{clean}}(x)$.

\paragraph{1. Establishing the Conditional Markov Chain}
The data-generating process described in \Cref{sec:information_cost} forms a conditional Markov chain for any fixed $X=x$:
$$
M\ \to\ Y\ \to\ Y^n \quad \text{conditioned on } X=x.
$$
This holds because:
\begin{enumerate}
    \item The hypothesis $M$ (which represents the true data-generating process $\eta_m(x)$) determines the distribution for the clean label $Y$.
    \item The noisy label $Y^n$ is generated based *only* on the clean label $Y$ (via the noise channel $T(x)$), without any direct influence from $M$.
\end{enumerate}
Therefore, given $Y$, the noisy label $Y^n$ is conditionally independent of the hypothesis $M$. This conditional independence implies:
\begin{equation}
\label{eq:app_cond_indep}
    I(M; Y^n \mid Y, X=x) = 0.
\end{equation}

\paragraph{2. Applying the Chain Rule for Mutual Information}
We decompose the joint mutual information $I(M; Y, Y^n \mid X=x)$ in two different ways using the chain rule:

\begin{itemize}
    \item \textbf{Decomposition 1 (Grouping with Y):}
    \begin{align}
    I(M; Y, Y^n \mid X=x) &= I(M; Y \mid X=x) + \underbrace{I(M; Y^n \mid Y, X=x)}_{\text{Equals } 0 \text{ by \Cref{eq:app_cond_indep}}} \nonumber \\
    &= I(M; Y \mid X=x) \nonumber \\
    &= I_{\mathrm{clean}}(x) \label{eq:app_info_1}
    \end{align}

    \item \textbf{Decomposition 2 (Grouping with $Y^n$):}
    \begin{align}
    I(M; Y, Y^n \mid X=x) &= I(M; Y^n \mid X=x) + I(M; Y \mid Y^n, X=x) \nonumber \\
    &= I_{\mathrm{noisy}}(x) + I(M; Y \mid Y^n, X=x) \label{eq:app_info_2}
    \end{align}
\end{itemize}

\paragraph{3. Final Derivation}
By equating \Cref{eq:app_info_1} and \Cref{eq:app_info_2}, we have:
\[
I_{\mathrm{clean}}(x) = I_{\mathrm{noisy}}(x) + I(M; Y \mid Y^n, X=x)
\]
By the non-negativity property of mutual information, the final term must be non-negative:
\[
I(M; Y \mid Y^n, X=x) \ge 0
\]
Therefore, we conclude:
\[
I_{\mathrm{clean}}(x) \ge I_{\mathrm{noisy}}(x) \quad \text{or} \quad \mathbf{I_{\mathrm{noisy}}(x) \le I_{\mathrm{clean}}(x)}
\]
Furthermore, the inequality is strict, $I_{\mathrm{noisy}}(x) < I_{\mathrm{clean}}(x)$, if $I(M; Y \mid Y^n, X=x) > 0$. This holds for any non-trivial noise channel $T(x)$ (i.e., not an identity or permutation matrix), as the noisy label $Y^n$ becomes a statistically lossy proxy for the true label $Y$.

\section{Experiment Details}
\label{supp:experiments_details}

\subsection{Dataset Details}

\paragraph{CIFAR-10/CIFAR-100}
Both datasets consist of 50,000 training images. Following established conventions, we evaluate performance under three standard noise models:
\begin{itemize}
    \item \textbf{Symmetric Noise:} Labels are randomly flipped to any of the other classes with a uniform probability.
    \item \textbf{Asymmetric Noise:} Labels are flipped to mimic real-world mistakes between visually similar categories (e.g., Horse $\leftrightarrow$ Deer and Dog $\leftrightarrow$ Cat).
    \item \textbf{Instance-Dependent Noise (IDN):} To approximate IDN without the prohibitive cost of a unique transition matrix for every sample, we use a grouped setting. The dataset is divided into 50 groups, with each group being assigned a different, randomly generated, diagonally-dominant transition matrix $T$.
\end{itemize}
We test a comprehensive range of noise levels: 20\%, 50\%, 80\%, and 90\% for symmetric noise, and 40\% for asymmetric noise.

\paragraph{Clothing1M}
Clothing1M~\cite{clothing1mdataset} is a large-scale, real-world benchmark for LNL. It contains 1 million images in 14 classes, crawled from online shopping websites. The dataset has a substantial level of intrinsic noise, estimated at approximately 38.5\%.

\subsection{Implementation Details}

\paragraph{CIFAR-10/CIFAR-100}
We use a PreActResNet-18~\cite{preactresnet} backbone for all CIFAR experiments.
\begin{itemize}
    \item \textbf{Standard Training:} The network is trained for 120 epochs using SGD with a weight decay of 5e-4 and a batch size of 128. The initial learning rate is 0.02, decaying to 0.002 at 60 epochs and 0.0002 at 80 epochs.
    \item \textbf{Long Training (for \Cref{fig:ideal_failure}):} To observe the long-term collapse, we train for 800 epochs. The learning rate schedule is extended, decaying at 400 and 600 epochs.
\end{itemize}

\paragraph{Clothing1M}
Following the setup of~\cite{dividemix}, we use a ResNet-50 backbone pretrained on ImageNet. The network is trained for 150 epochs with a weight decay of 1e-3 and a batch size of 32. The initial learning rate is 0.002, decaying by a factor of 10 at 50 and 100 epochs.

\begin{figure}[tbp]
    \centering
\includegraphics[width=0.9\linewidth]{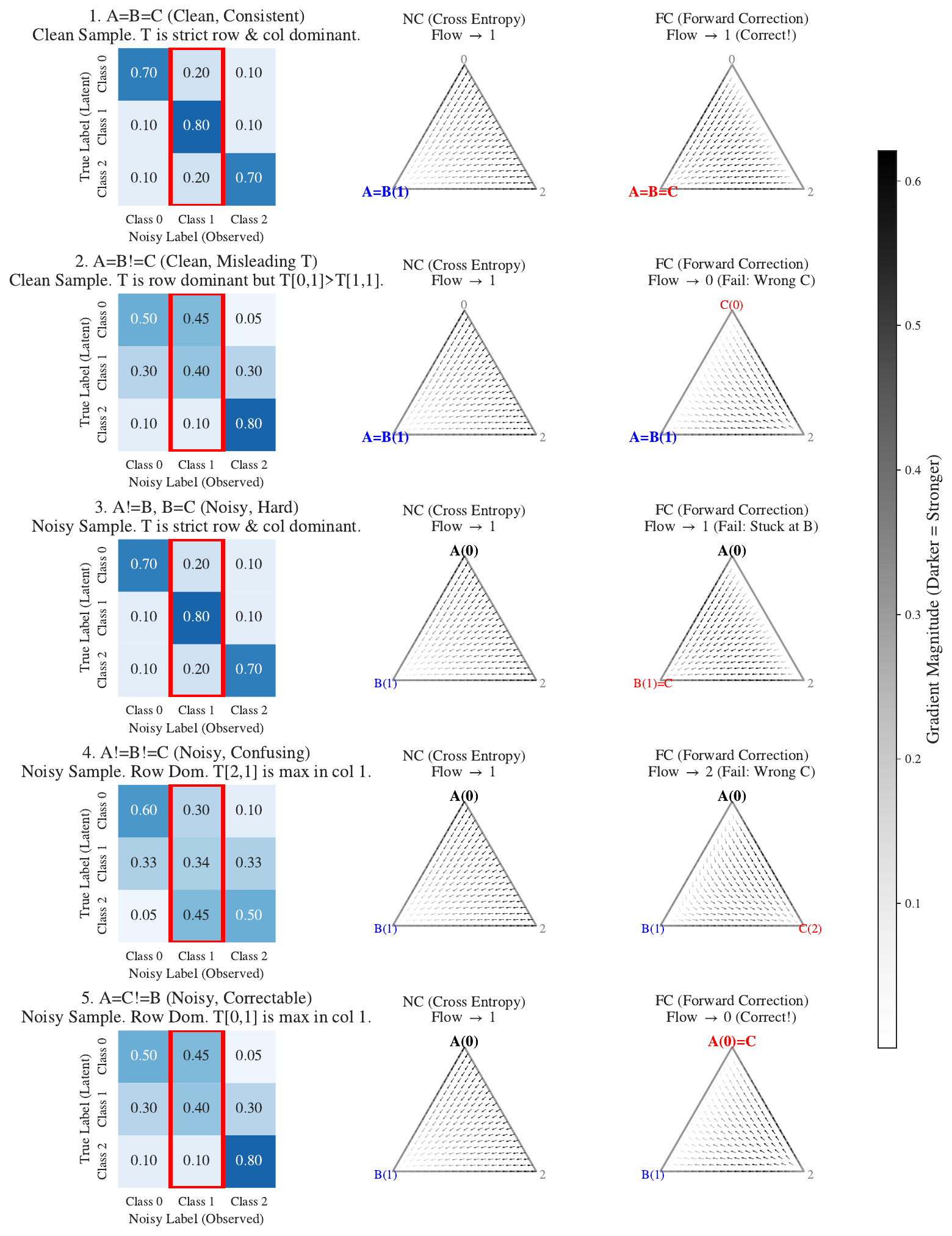}
    \caption{Gradient vector field of the FC loss on a 3-class simplex. 
We denote the clean label vertex as $A$ ($\mathbf{e}_{y^*}$), the noisy label as $B$ ($\mathbf{e}_{y^n}$), and the theoretical FC optimum as $C$ ($\mathbf{e}_{k_{\fc}^*}$). 
The vector field confirms that the global minimum is at $C$. 
However, the noisy vertex $B$ acts as a strong, non-optimal attractor. 
The vanishing gradient magnitude (``dead zone'') near $B$ traps SGD, leading to the 'pseudo-convergence' analyzed in \Cref{app:proof_gradient_flow}.
}
    \label{fig:gradient}
\end{figure}

\end{document}